\documentclass[letterpaper]{article} % DO NOT CHANGE THIS
\usepackage{aaai2026}  % DO NOT CHANGE THIS
\usepackage{times}  % DO NOT CHANGE THIS
\usepackage{helvet}  % DO NOT CHANGE THIS
\usepackage{courier}  % DO NOT CHANGE THIS
\usepackage[hyphens]{url}  % DO NOT CHANGE THIS
\usepackage{graphicx} % DO NOT CHANGE THIS
\usepackage{natbib}  % DO NOT CHANGE THIS AND DO NOT ADD ANY OPTIONS TO IT
\usepackage{caption} % DO NOT CHANGE THIS AND DO NOT ADD ANY OPTIONS TO IT
\usepackage{algorithm}
\usepackage{algpseudocode}
\usepackage{amssymb,amsthm}
\newtheorem{proposition}{Proposition}
\usepackage{newfloat}
\usepackage{listings}
\usepackage{amsfonts}
\usepackage{amsmath}
\usepackage{xcolor}
\DeclareCaptionStyle{ruled}{labelfont=normalfont,labelsep=colon,strut=off} % DO NOT CHANGE THIS
\floatstyle{ruled}
\newfloat{listing}{tb}{lst}{}
\floatname{listing}{Listing}
\usepackage{booktabs}
\usepackage{adjustbox}
\usepackage{multirow}
\newcommand{\model}{\textsc{RTFM}}
\newcommand{\modelfull}{\textsc{Robust Tabular Foundation Models}}
\title{Robust Tabular Foundation Models}
\author{
    Matthew Peroni\textsuperscript{\rm 1,}\textsuperscript{\rm 2},
    Franck Le \textsuperscript{\rm 2},
    Vadim Sheinin \textsuperscript{\rm 2}
}
\affiliations{
    \textsuperscript{\rm 1} MIT \;
    \textsuperscript{\rm 2} IBM Research\\
     mperoni1@mit.edu, fle@us.ibm.com, vadims@us.ibm.com
}

\begin{document}

\maketitle

\begin{abstract}
The development of tabular foundation models (TFMs) has accelerated in recent years, showing strong potential to outperform traditional ML methods for structured data. A key finding is that TFMs can be pretrained entirely on synthetic datasets, opening opportunities to design data generators that encourage desirable model properties. Prior work has mainly focused on crafting high-quality priors over generators to improve overall pretraining performance. Our insight is that parameterizing the generator distribution enables an adversarial robustness perspective: during training, we can adapt the generator to emphasize datasets that are particularly challenging for the model. We formalize this by introducing an optimality gap measure, given by the difference between TFM performance and the best achievable performance as estimated by strong baselines such as XGBoost, CatBoost, and Random Forests. Building on this idea, we propose \modelfull\ (\model), a model-agnostic adversarial training framework. Applied to the TabPFN V2 classifier, \model\ improves benchmark performance, with up to a 6\% increase in mean normalized AUC over the original TabPFN and other baseline algorithms, while requiring less than 100k additional synthetic datasets. These results highlight a promising new direction for targeted adversarial training and fine-tuning of TFMs using synthetic data alone.
\end{abstract}


\section{Introduction}
Recent survey studies have challenged the deep learning community to improve upon boosted tree methods for learning from structured data, having found a significant gap in performance remained between the paradigms \cite{grinsztajn_why_2022}. In response, substantial progress has been made to close this gap, resulting in a variety of strong deep learning approaches being proposed over the past few years \citep{somepalli_saint_2021, gorishniy_revisiting_2023, hegselmann_tabllm_2023, hollmann_tabpfn_2023, carballo_tabtext_2023, yan_making_2024, link_mitra_2025}. Among these new approaches, tabular foundation models (TFMs) which rely on in-context learning (ICL) \citep{hollmann_tabpfn_2023, qu_tabicl_2025, link_mitra_2025} have emerged as a promising direction for classification and regression tasks with structured datasets. The benefits of this approach are two-fold. In many cases, the performance of the TFMs improves upon strong traditional baselines such as boosted trees \citep{friedman_greedy_2001, erickson_tabarena_2025}. Additionally, this performance is achieved in a zero-shot framework, enabling high-quality predictions on new datasets in milliseconds when GPU-accelerated. Further, as we will explore in this work, since these TFMs are pretrained using synthetic data, there is significant opportunity for improvement as the data generation process is expanded and improved. 

Training TFMs relies on generating a large amount of diverse synthetic datasets \citep{hollmann_tabpfn_2023}. As we will formally define in the following section, this generation process relies on constructing structural causal models (SCMs) from which datasets can be sampled \citep{pearl_causality_2009, shanmugam_elements_2018}. The structure of these SCMs is implicitly parameterized, giving significant control over the data generation process. All current publicly available, competitive TFMs  \citep{hollmann_tabpfn_2023, qu_tabicl_2025, link_mitra_2025} have been pretrained on datasets generated from a fixed prior distribution over these SCM parameters. However, fixed priors underrepresent certain regions of the parameter space, potentially degrading performance on real-world datasets with similar structure. This contributes to state-of-the-art TFMs still lagging behind tree-based methods on some benchmarks \citep{mcelfresh_when_2024, ye_closer_2025}. In this work, we leverage the significant control provided by the data generation process to frame TFM training from an adversarial robustness perspective \citep{madry_towards_2019}. Recent work by \cite{wu_zero-shot_2025} also explores a narrower version of this idea, focusing on adjusting the weights of a specific class of SCMs in GAN-style training of TFMs. In this work, we focus on a broad and highly flexible framework for adversarially robust training of TFMs. Our contributions can be summarized as: \textbf{1)} We formalize adversarial training over the SCM parameter space, allowing the model to adapt to challenging regions (e.g., varying numbers of features, categorical feature ratios, nonlinearities). We introduce an \textit{optimality gap} concept and use it to target regions where the TFM underperforms relative to the best achievable performance. \textbf{2)} We propose an efficient, model-agnostic two-stage adversarial training algorithm for TFMs, which we call \modelfull\ (\model). \textbf{3)} We apply \model\ to TabPFN V2 \citep{hollmann_accurate_2025}, showing that with only 90k additional training datasets, we can significantly improve the ranking of TabPFN on several real-world tabular benchmarks.

\section{Problem Definition}
\begin{figure*}[t!]
    \centering
    \includegraphics[width=0.90\linewidth]
    {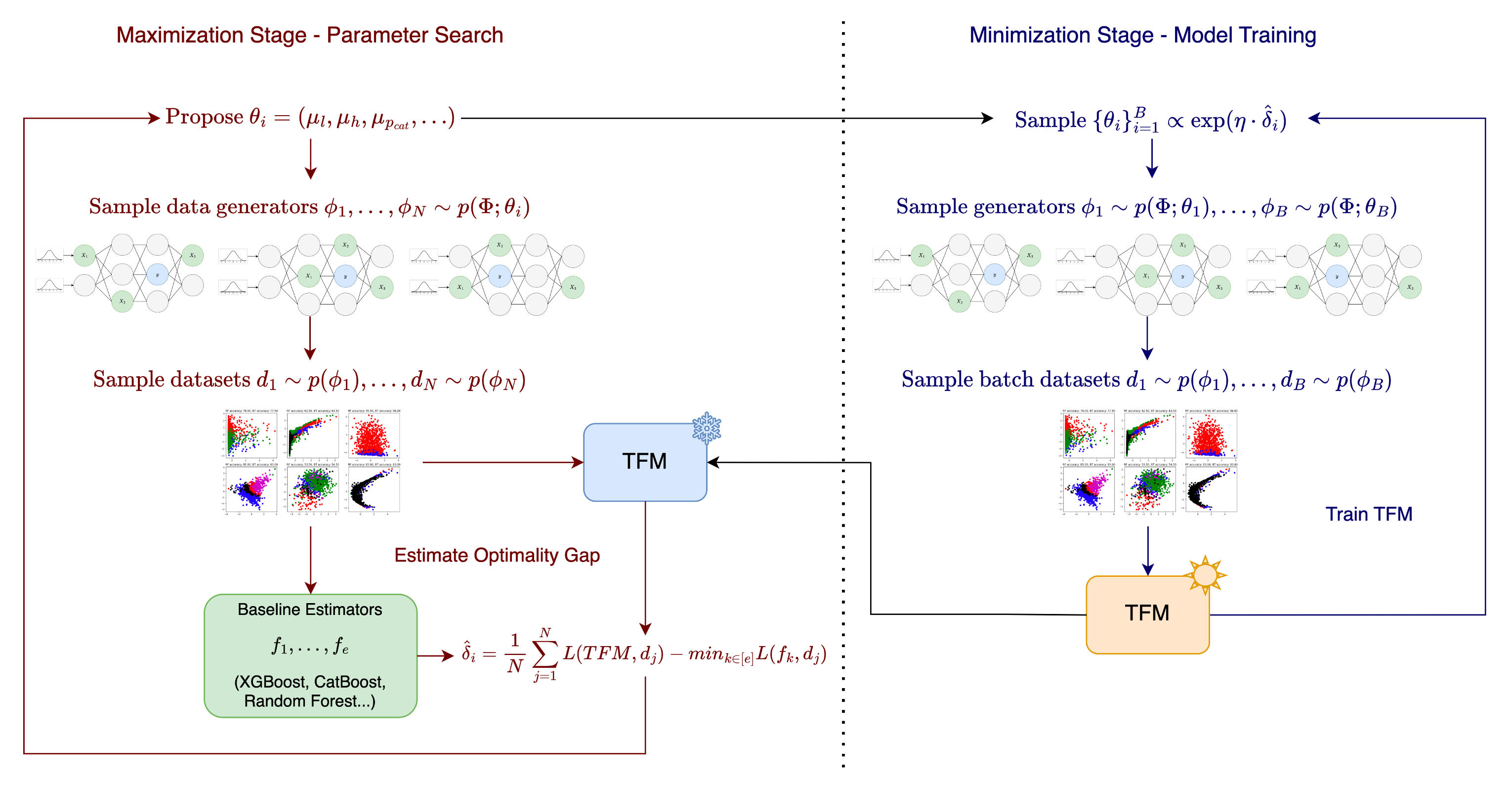}
    %{images/diagrams/rtfm_diagram_sampling.drawio.png}
    \caption{Overview of \modelfull\ (\model).}
    \label{fig:rtfm}
\end{figure*}
We begin by defining some notation, and then introduce our theoretical formulation for \model. Let $\Phi$ represent the space of hypothesis functions which define the relationship between features $x$ and outcome $y$. The outcome $y$ can be discrete or continuous, but for this work, we focus on classification tasks. Then a hypothesis $\phi \in \Phi$ is a specific function which we can treat as a data generating mechanism. In the case of SCMs, $\Phi$ is the entire space of SCMs, and $\phi \in \Phi$ is a specific instance of an SCM, from which a dataset $D \sim p(\phi)$ can be generated. In this work, we implement SCMs as randomly initialized multi-layer perceptrons (MLPs), as introduced in \cite{hollmann_tabpfn_2023}. Full details of this process are provided in Appendix \ref{appendix:scm}. For example, we can sample the number of layers $l \sim p(l)$, the hidden size $h \sim p(h)$, the activation function $a \sim p(a)$, and the ratio of categorical features $r_{cat} \sim p(r_{cat})$, among other dimensions. An SCM is generated by first sampling from the joint distribution over these hyperparameters, each of which is parameterized. For example, we might sample the ratio of categorical features $r_{cat} \sim TruncNorm(\mu_{r_{cat}}, a=0, b=1)$ from the truncated normal distribution with some mean $\mu_{r_{cat}}$ which parametrizes the distribution. Then, assuming we have similar parameterizations for the distribution over each hyperparameter, we define $\theta = (\mu_l, \mu_h, \mu_{r_{cat}}, \dots)$ as the parameterization of the joint distribution over all the hyperparameters. We define $\mathcal{P}$ as the space of the parameters, such that $\theta \in \mathcal{P}$. The generators are then sampled from $\phi \sim p(\Phi;\theta)$.

 We now define the original objective for the Prior-Fitted Network (PFN) learning task and provide our adversarially robust formulation. Let $g_{\mathbf{W}}$ be our predictive model parameterized by weight matrix $\mathbf{W}$. As a PFN, $g_{\mathbf{W}}$ takes as input a sequence of labeled training samples and unlabeled test samples, and predicts labels for the test samples. For a given data generator $\phi \in \Phi$, the PFN loss is defined as the cross entropy loss over the test samples,
\begin{align}
% \begin{split}
    \mathcal{L}_{PFN}(\mathbf{W}; \phi) = \\
&\mathbb{E}_{(\{(x_{t},y_{t})\} \cup D) \sim p(\phi)}[-\log(g_{\mathbf{W}}(y_{t}|x_{t}, D)]. \nonumber
% \end{split}
\end{align}
Since the model will be trained by sampling many data generators from the distribution $\phi \sim p(\Phi ; \theta)$, we can write the original optimization problem as follows,
\begin{equation}
    \min_{\mathbf{W}} \mathbb{E}_{\phi \sim p(\Phi ; \theta)} [\mathcal{L}_{PFN}(\mathbf{W}; \phi)].
\end{equation}
Since our goal is to learn a foundation model which performs well across all data generation schemes, we want to introduce an adversary who can shift the distribution over data generators to identify the regions over which $g_{\mathbf{W}}$ yields the worst performance. However, there is a key subtlety in this formulation. Maximizing the loss alone could lead the maximizer to select data generators which produce datasets that are hard to learn for any model. Instead, we want to maximize the \textit{optimality gap} between the model's current performance and the best achievable performance. Since we are using the cross-entropy loss, we can frame this optimality gap in terms of the conditional entropy of the data. For any data generator $\phi$, let us denote the features as a random variable $Z_x$ and the target as $Z_y$ such that $(Z_x, Z_y) \sim p(\phi)$. Then the Bayes' optimal predictor is exactly $f^* = p(Z_y | Z_x)$ which achieves the cross-entropy loss $L_{CE}(f^*) = H_{\phi}(Z_y | Z_x)$, the conditional entropy of the target $Z_y$ given the features $Z_x$. It follows that for any fixed data generator, $\mathcal{L}_{PFN} \geq H(Z_y | Z_x)$, and $\mathcal{L}_{PFN} - H(Z_y | Z_x) \geq 0$ is the optimality gap. We can then state the adversarially robust optimization problem in terms of maximizing the optimality gap,
\begin{equation}\label{eq:main-objective}
    \min_{\mathbf{W}}\max_{\theta \in \mathcal{P}} \mathbb{E}_{\phi \sim p(\Phi ; \theta)} [\mathcal{L}_{PFN}(\mathbf{W}; \phi) - H_{\phi}(Z_y | Z_x)].
\end{equation}
In practice, maximizing over a single parameterization could lead to the model overfitting to a particular region of the parameter space. We therefore frame this problem from a distributionally robust optimization (DRO) perspective by constraining the actions of the maximizer \cite{rahimian2019distributionally}. Rather than a specific parameterization $\theta$, we allow the maximizer to select a distribution $Q \in \Delta_{\mathcal{P}}$ over parameters, where $\Delta_{\mathcal{P}}$ is the simplex over $\mathcal{P}$. We assume $\mathcal{P}$ is discrete, which is reasonable in practice since the parameters are either categorical or can be easily discretized (e.g. $\mu_{r_{cat}} \in \{0, 0.1, \dots, 1\}$). For this work, we do not assume access to a reference distribution over $\mathcal{P}$ from which we can constrain the divergence of $Q$. Instead, we will constrain the minimum entropy, $H_{min}$, of $Q$ directly. This is equivalent to constraining the KL-divergence of $Q$ when the reference distribution is uniform. For brevity, we denote the objective function as $\delta_\theta(\mathbf{W}) = \mathbb{E}_{\phi \sim p(\Phi ; \theta)} [\mathcal{L}_{PFN}(\mathbf{W}; \phi) - H_{\phi}(Z_y | Z_x)]$. We can then define the DRO problem as
\begin{equation}\label{eq:dro-objective}
    \min_{\mathbf{W}}\max_{\substack{
        Q \in \Delta_{\mathcal{P}} \\ H(Q) \geq H_{min}}} \mathbb{E}_{\theta \sim Q}[\delta_\theta(\mathbf{W})].
\end{equation}
As we prove in Appendix \ref{sec:proofs}, the optimal distribution for the maximizer in Equation \ref{eq:dro-objective} is a softmax distribution. Specifically, denoting $Q^* = (q_1^*,\dots,q_{|\mathcal{P}|}^*)$ as the optimal distribution, for any $i \in |\mathcal{P}|$, $q_i^* \propto \exp(\eta \cdot \delta_{\theta_i}(\mathbf{W}))$. The temperature $\eta  = \eta(H_{min}, \delta_\theta(\mathbf{W}))$ is uniquely determined by the minimum entropy and the optimality gaps. The DRO problem can then be rewritten as
\begin{equation}\label{eq:dro-simplified}
    \min_{\mathbf{W}} \sum_{i=1}^{|\mathcal{P}|} 
    \frac{\exp(\eta \cdot \delta_{\theta_i}(\mathbf{W}))}{\sum_{i \in |\mathcal{P}|}\exp(\eta \cdot \delta_{\theta_i}(\mathbf{W}))} \cdot
    \delta_{\theta_i}(\mathbf{W}).
\end{equation}
While Equation \ref{eq:dro-simplified} is the precise theoretical optimization problem we wish to solve, in reality, estimating the conditional entropy $H_{\phi}(Z_y | Z_x)$ or evaluating it analytically is often infeasible. Therefore, rather than estimate $H_{\phi}(Z_y | Z_x)$ directly, in this work, we replace this term with the minimum cross-entropy loss over a collection of baseline estimators $f_1, \dots, f_e$, such as XGBoost, Catboost, and Random Forests \cite{chen_xgboost_2016,dorogush_catboost_2018, breiman_random_2001}, which are fit to each generated dataset. Since $\min_{i \in [e]}\mathcal{L}_{PFN}(f_i;\phi) \geq H(Z_y | Z_x)$, it follows that $\mathbb{E}_{\phi \sim p(\Phi;\theta)}[\mathcal{L}_{PFN}(\mathbf{W}; \phi) - \min_{i \in [e]} \mathcal{L}_{PFN}(f_i; \phi)] \leq \delta_\theta(\mathbf{W})$, providing us with a lower bound on the optimality gap. We then modify our original objective function, defining 
\begin{equation}
    \widehat\delta_\theta(\mathbf{W}) = \mathbb{E}_{\phi \sim p(\Phi;\theta)}[\mathcal{L}_{PFN}(\mathbf{W}; \phi) - \min_{i \in [e]} \mathcal{L}_{PFN}(f_i; \phi)],
\end{equation} 
and arriving at the following optimization problem,
\begin{equation}\label{eq:general-robust-opt}
    \min_{\mathbf{W}} \sum_{i=1}^{|\mathcal{P}|} 
    \frac{\exp(\eta \cdot \widehat\delta_{\theta_i}(\mathbf{W}))}{\sum_{i \in |\mathcal{P}|}\exp(\eta \cdot \widehat\delta_{\theta_i}(\mathbf{W}))} \cdot
    \widehat\delta_{\theta_i}(\mathbf{W}).
\end{equation}
Since $\widehat \delta_\theta(\mathbf{W}) \leq \delta_\theta(\mathbf{W})$, it follows that $\widehat\delta_\theta(\mathbf{W}) > 0$ implies $\delta_\theta(\mathbf{W}) > 0$, such that for any parameter estimated to have a positive optimality gap, the model $g_{\mathbf{W}}$ does indeed have room for improvement. Further, this lower bound can always be improved by adding more baseline estimators with varying strengths and inductive biases.
We aim to solve Equation \ref{eq:general-robust-opt} to recover a robust tabular foundation model. As a whole, this problem is both non-convex and non-differentiable, including a composition of multiple black-box functions. In the following section, we introduce a two-stage optimization approach to find high-quality solutions to this problem.
% Notice that the minimization over $\mathbf{W}$ is outside the expectation over data generators while the minimization over $f_1,\dots,f_k$ is within the expectation. This reflects the Tab-PFN model being a foundation model, which aims to minimize the loss over a distribution of data generators, while the individual boosted tree models are fit to specific data generators. We can further generalize this formulation by estimating $H(Z_y | Z_x)$ using a collection of baseline estimators $f_1, \dots f_e$, taking the minimum cross-entropy loss over all of these estimators as follows,

%\notePH{Suppose we refer to eq(4) above as regret. I think we could setup the RTFM outer loop or parameter selection as one to reduce the total cum-regret over a horizon. So, it would be good to show plots of this cum-reg as opposed to just instantaneous regret which is also useful. Sampling probability, was previously: $\theta_i^t = argmax_{\theta_i}  \delta_i(\theta_i)$. Now in the figure we have $\theta_i^t \propto \delta_i(\theta_i)$. The IGW version is close it is is just set up as: $\theta_i^t \propto 1/\left[\alpha + \gamma(\delta^* - \delta_i(\theta_i))\right] $ where $ \delta^* = \max_{\theta_i}  \delta_i(\theta_i)$. Let me think about how do we go about setting up the maximization stage as we need a representative set of parameters and not just those obtained in the path of finding the optimal.}

\section{\modelfull}\label{sec:rtfm}
\begin{table*}[t]
\centering
\caption{Experiment results across two tabular dataset benchmarks. For consistency, we report normalized AUC, with values scaled to [0,1] for each dataset \cite{mcelfresh_when_2024,hollmann_accurate_2025}. For rank-1 wins, we do not count ties, only outright wins.}
\setlength{\tabcolsep}{2pt} % reduce horizontal padding
% \small % smaller font size
\begin{adjustbox}{max width=1\textwidth} % scale table if still too wide
\begin{tabular}{llccccccc}
\toprule
 &\textbf{Metric} & \textbf{Log. Reg.} & \textbf{MLP} & \textbf{Random Forest} & \textbf{CatBoost} & \textbf{XGBoost} & \textbf{TabPFN\textsubscript{n.e.}} & \textbf{TabPFN\textsubscript{n.e.} (\model)} \\
\midrule
\multirow{3}{*}{\textbf{TabPertNet}\cite{ye_towards_2024}} 
& Mean rank AUC & 5.1 & 4.6 & 4.0 & 3.8 & 4.6 & 3.2 & \textbf{2.7} \\
& Mean Norm. AUC  &$0.4253 $&$0.5005$& $0.6481$ & $0.6663$ & $0.5222$ & $0.7483$ & $\mathbf{0.8167}$ \\
& Rank-1 Wins & 1 & 8 & 5 & 7 & 5 & 11 & \textbf{17}\\
\midrule
\multirow{3}{*}{\textbf{TabArena}\cite{erickson_tabarena_2025}} 
& Mean rank AUC OVO & 4.9 & 6.3 & 4.8 & 3.4 & 4.5 & 2.2 & \textbf{1.9} \\
& Mean Norm. AUC OVO &$0.4277$& $0.1801$ & $0.5761$ & $0.7749$ & $0.5918$ & $0.9031$ & $\mathbf{0.9298}$ \\
& Rank-1 Wins & 2 & 0 & 0 & 2 & 0 & 5 & \textbf{12}\\
\bottomrule
\end{tabular}
\end{adjustbox}
\label{table:main}
\end{table*}
In this section, we give an overview of our two-stage optimization algorithm, \modelfull\ (\model), which is visualized in Figure \ref{fig:rtfm}. The pseudocode and additional details are provided in Appendix \ref{sec:appendix-alg}. We first decompose the problem naturally into a maximization stage (parameter search) and a minimization stage (model training). We assume access to a pretrained TFM, $g_{\mathbf{W}}$, with weights $\mathbf{W}$ which will be used as the input model for our \model\ algorithm. While not absolutely necessary, starting with a pretrained model gives us a warm-start for the model weights. In addition, we specify a fixed set of baseline models $f_1,\dots,f_e$ which will be used to estimate the optimality gap. We begin with the maximization stage, during which we freeze $g_{\mathbf{W}}$ to maximize the optimality gap in its current state. Since the cardinality of the parameter space is large in practice, we use a black-box optimization algorithm \cite{akiba_optuna_2019, watanabe_tree-structured_2023, bergstra_algorithms_2011} to efficiently search the space for parameters with large optimality gaps. 

Using the black-box optimization algorithm, parameters $\theta_i$ are proposed in sequence. For each $\theta_i$ proposed, we estimate the optimality gap by sampling a fixed number, $n_{ds}$, of generators $\phi_1,\dots, \phi_{n_{ds}} \sim p(\Phi;\theta_i)$ and datasets $d_j\sim p(\phi_j)$ and compute $\widehat{\delta}_{\theta_i} = \frac{1}{n_{ds}}\sum_{j=1}^{n_{ds}}\mathcal{L}_{PFN}(\mathbf{W}; d_j) - \min_{k \in [e]}\mathcal{L}_{PFN}(f_k, d_j)$. Then for each proposed $\theta_i$, we must fit $n_{ds} \cdot e$ baseline estimators. However, each pair of dataset and baseline estimator can be fit independently, such that this step can be trivially parallelized, where each baseline estimator is fit to each dataset in parallel. In practice, we found that for $n_{ds} = 20$ and $e = 7$, the estimated optimality gap $\widehat{\delta}_{\theta_i}$ could be computed in a matter of seconds when parallelized, given sufficient CPU cores (we used $n_{cores} = n_{ds} \cdot e$). We repeat this process for a fixed number of trials, $n_{trials}$, with the underlying optimization algorithm proposing new parameters $\theta_i$ based on the observed estimated optimality gaps. At the end of this stage, we have a collection $\{(\theta_i, \delta_{\theta_i})\}_{i=1}^{n_{trials}}$ of parameters and associated estimated optimality gaps.

In the minimization stage, our goal is to train $g_{\mathbf{W}}$ on the regions of the parameter space with large optimality gaps. Following our derivation in Equation \ref{eq:general-robust-opt}, we define a distribution $Q$ over $\mathcal{P}$ such that $q_i \propto \exp(\eta \cdot \delta_{i}(\mathbf{W}))$.  The function \(\eta \mapsto H(Q(\eta))\) is strictly decreasing on $(0, \infty)$, so we can quickly find $\eta$ given a set $H_{min}$ via one-dimensional root finding (e.g., bisection). Using the observations from the maximization stage, our optimization problem becomes
\begin{equation}\label{eq:general-robust-opt-trials}
    \min_{\mathbf{W}} \sum_{i=1}^{n_{trials}} 
    \frac{\exp(\eta \cdot \widehat\delta_{\theta_i}(\mathbf{W}))}{\sum_{i \in [n_{trials}]}\exp(\eta \cdot \widehat\delta_{\theta_i}(\mathbf{W}))} \cdot
    \widehat\delta_{\theta_i}(\mathbf{W}).
\end{equation}
% Rather than taking a sum over all $|\mathcal{P}|$ possible parameterizations as in Equation \ref{eq:general-robust-opt}, we take the sum over the $n_{trials}$ parameters that were proposed during the maximization stage. 
Then, to train the TFM, we generate each dataset within a batch by sampling $\theta_i \sim Q$ and generator $\phi \sim p(\Phi; \theta_i)$. This sampling and dataset generation process is repeated for each batch during model training such that, in expectation, the loss is weighted according to Equation \ref{eq:general-robust-opt-trials}. After training $g_{\mathbf{W}}$ for $n_{iter}$ iterations, we freeze the model return to the maximization stage, using the updated model weights to generate the next set of estimated optimality gaps. This max-min iteration is repeated until the objective in Equation \ref{eq:general-robust-opt-trials} converges. In addition to the baseline estimators $f_1,\dots,f_e$, after a fixed number of max-min iterations, we incorporate the TFM with its original weights as an additional baseline model. This can help mitigate any unlearning that has occurred, encouraging the model to both improve upon strong baselines as well as its own performance at the start of training.

\section{Experiments}

We evaluate our \model\ algorithm on benchmark classification datasets from TabArena \cite{erickson_tabarena_2025} and TabPertNet \cite{ye_towards_2024}, using TabPFN V2 \cite{hollmann_accurate_2025} as the starting point for our TFM. We provide a full description of the experiment setup and results in Appendix \ref{appendix:results}. For the \model\ algorithm, we use $n_{trials}=100$ and $n_{ds}=20$ datasets to estimate the optimality gap for each trail. During the minimization stage, we use a learning rate of $lr=1e\text{-}5$, a batch size of $b=64$, and train the model for $n_{iter}=3000$ training steps per iteration. We run the full max-min loop for $e=30$ epochs. For baseline models, we use Random Forest, CatBoost, XGBoost, Logistic Regression, and MLPs. For this work, we evaluate the performance of each algorithm in Table \ref{table:main} using their respective default settings, and report TFM results without ensembling. All experiments were run on a single node with one A100 GPU and 256 CPU cores over which we distributed the training and evaluation of baseline estimators during the maximization stage. Details of the parameter space are provided in Appendix \ref{appendix:scm}.

The results of our experiments are summarized in Table \ref{table:main}, and full results are provided in Appendix \ref{appendix:results}. Overall, we observe a significant improvement in both AUC and mean rank when applying the \model\ algorithm compared to both the baseline TabPFN model, as well as a the collection of baseline estimators. To measure the significance of our results, we first perform the Friedman test for repeated samples on the median normalized AUC scores \cite{demvsar2006statistical, benavoli2016should}. For the TabPertNet datasets, the Friedman test yielded $p=2.2 \times 10^{-14}$, and for TabArena, the test yielded $p=1.6\times10^{-12}$. Having established significance for the benchmark performance varying by model class, we performed the Wilcoxon signed-rank test pair-wise, comparing the RTFM version of TabPFN to the other models \cite{conover1999practical}. For TabPertNet, the Wilcoxon test comparing TabPFN (RTFM) and the original TabPFN model yielded $p=0.0023$, and for TabArena, the test yielded $p=0.0103$. The other Wilcoxon tests all yielded smaller p-values, supporting the significance of applying RTFM to improving TabPFN against all the baseline estimators.These results support the statistical significance of the performance improvements observed from applying the RTFM algorithm. 

Across our three metrics, we observe that TabPFN (RTFM) strictly dominates the other models. Further, among datasets where the original TabPFN model is not the top ranked model (rank $>$ 1), we observe consistent improvement from applying the RTFM algorithm. For example, for TabPertNet, among datasets where TabPFN (original) has rank $>$ 2, the mean rank of TabPFN (RTFM) is 3.26, which is the highest mean rank among all models and a 1.24 point improvement over the original. We also find that among 21\% of these datasets where TabPFN (original) is strictly worse than the baseline models, TabPFN (RTFM) ``leaps'' the competition and is the top ranked model. We observe similar gains in mean rank for the TabArena datasets, with a 1.8 point improvement in mean rank on datasets where TabPFN (original) has rank $>$ 1 and TabPFN (RTFM) ``leaps'' the competition for 20\% of these datasets. These results demonstrate the ability of the \model\ algorithm to not only improve the performance of TFMs overall, but to address specific types of datasets where the TFM lags behind traditional methods, in some cases becoming the dominant model. We also note that these results were achieved while training the TabPFN model on an additional $90k$ synthetic datasets, about 1\% of the $>9$ million datasets generated to pretrain the original TabPFN model.

\section{Conclusion}
In this work, we introduce an exciting new direction for training TFMs which is both model-agnostic and massively extensible. We demonstrate its ability to improve upon an existing state-of-the-art TFM with limited additional training across multiple tabular benchmarks. While we use TabPFN as our starting point, \model\ can be applied to any TFM, such as Mitra or TabICL. Our framework can also be easily extended to regression tasks. Further, we achieve these results while only implementing MLP-based SCMs. In future work, plan to expand the parameter space and include tree-based SCMs within our framework. The \model\ algorithm is highly-parallelizable, enabling large-scale training and efficient parameter search. We hope these promising results motivate further research in the direction of optimized, adaptive synthetic dataset generation for pretraining and fine-tuning of TFMs.

% \section*{Acknowledgments}
% Acknowledgments

\bibliographystyle{plain}
\bibliography{aaai2026}

%%%%%%%%%%%%%%%%%%%%%%%%%%%%%%%%%%%%%%%%%%%%%%%%%%%%%%%%%%%%

\appendix
\setcounter{secnumdepth}{1}
\section{\model\ Algorithm}\label{sec:appendix-alg}
We summarize \modelfull\ in Algorithms \ref{alg:rtfm} and \ref{alg:parameter-search}. A description of the algorithm is given in Section \ref{sec:rtfm}. In Algorithm \ref{alg:parameter-search}, we reference an abstract function SuggestParameters which represents any black-box optimization algorithm. In this case, we use the Optuna python package \cite{akiba_optuna_2019} and the Tree-structed Parzen Estimator \cite{watanabe_tree-structured_2023} algorithm, which belongs to the class of Bayesian Optimization algorithms \cite{frazier2018tutorial}. At each iteration, this algorithm uses the information gained from the previous trials (the estimated optimality gaps) to propose a new parameter, balancing exploration and exploitation to efficiently find parameters corresponding to large optimality gaps. Since we will sample the parameters following a softmax distribution, identifying the parameters corresponding to the largest optimality gaps is paramount, as the probability of sampling from the rest of the parameter space will be negligible. 
To begin the minimization stage, we first construct the distribution $Q \in \Delta_{\mathcal{P}}$. Rather than constructing $H_{min}$ as a parameter itself, we introduce a scalar parameter $c \in (0,1)$ and define $H_{min} = c\log(n_{trials})$, such that $H_{min}$ is a given as a fraction of the maximum possible entropy of $Q$. With $H_{min}$ specified, we can find $\eta$ such that $H(Q(\eta)) = H_{min}$ using one-dimensional root finding  (e.g. bisection). Then we generate each dataset in a training batch by sampling a parameter $\theta_i \sim Q$ and generator $\phi \sim p(\Phi; \theta_i)$. We train the model for a fixed number of iterations and then return to the maximization stage. In our implementation, after five epochs we incorporate the original TabPFN model as another baseline model. This can potentially address regions of the parameter space where the model has reduced its performance by focusing on improving in other regions. An illustration of the maximum estimated optimality gap found during each maximization stage is shown in Figure \ref{fig:opt-gap}. We observe that after the original TFM is introduced as an additional baseline model, the parameter search uncovers larger optimality gaps for a few epochs, suggesting that the model did have some performance degradation during the initial training epochs over some regions of the parameter space.

The \model\ algorithm is highly parallelizable, both in the dataset generation steps and the parameter search. For a fixed parameterization $\theta$, each dataset generator $\phi \sim p(\Phi; \theta)$ is sampled independently from the same distribution, such that all datasets generated for a given batch can be created concurrently. The ParameterSearch algorithm can be completely parallelized within each trial. In our implementation, for each trial, we generate a batch of datasets and then fit and evaluate each model and dataset pair concurrently. For example, in our experiments, we set $n_d=20$ and used $e=7$ baseline models, resulting in $n_{ds} * e = 140$ separate processes distributed across all available cores.
\begin{algorithm}[ht!]
\caption{\modelfull\ (\model)}\label{alg:rtfm}
\begin{algorithmic}
\Require Initial TFM weights $\mathbf{W}$, baseline estimators $\{f_1, \dots, f_e\}$, $n_e$, $n_{iter}$, $n_t$, $n_d$, $c \in (0,1)$
\State $\{(\theta_i, \delta_i)\}_{i=1}^{n_{trials}} \gets \text{ParameterSearch}(\mathbf{W}, \{f_1, \dots, f_e\}, n_t, n_d)$
\State $H_{min} \gets c\log(n_{t})$
\State $\eta \gets \text{RootFinding}(H_{min}, \delta)$
\State $Q \gets \text{Softmax}(\eta \cdot \delta)$
\For{$e \in [n_e]$}
\For{$t \in [n_{iter}]$} 
% split into train/test sets
\State $\theta_1,\dots,\theta_{n_b} \sim Q$
\State $\phi_1,\dots, \phi_{n_b} \sim p(\Phi ; \theta_1), \dots,p(\Phi;\theta_{n_b})$
\State $\{(x^i, y^i)\}_{i=1}^{n_b} \sim p(\phi_1),\dots, p(\phi_{n_b})$ 
\State $\{x^i_{train}, x^i_{test}, y^i_{train}, y^i_{test}\}_{i=1}^{n_b} \gets \text{TrainTestSplit}(\{(x^i, y^i)\}_{i=1}^{n_b})$
\State $l(g_{\mathbf{W}}) \gets \mathcal{L}_{PFN}(g_{\mathbf{W}}, \{x^i_{train}, x^i_{test}, y^i_{train}, y^i_{test}\}_{i=1}^{n_b})$
\State $\mathbf{W} \gets \text{update}(\mathbf{W}, l(g_{\mathbf{W}}))$ \Comment{Gradient step}
\EndFor
\State $\{(\theta_i, \delta_i)\}_{i=1}^{n_{trials}} \gets \text{ParameterSearch}(\mathbf{W}, \{f_1, \dots, f_e\}, n_t, n_d)$
\State $\eta \gets \text{RootFinding}(H_{min}, \delta)$
\State $Q \gets \text{Softmax}(\eta \cdot \delta)$
\EndFor
\State \textbf{Return} Robust TFM model $g_{\mathbf{W}}$
\end{algorithmic}
\end{algorithm}

\begin{algorithm}[ht!]
\caption{ParameterSearch}\label{alg:parameter-search}
\begin{algorithmic}
\Require TFM weights $\mathbf{W}$, baseline estimators $\{f_1, \dots, f_e\}$, num. trials $n_t$, num. datasets $n_d$
\State $D_{\theta} = \{\}$
\For{$i \in [n_t]$}
\State $\theta_i \gets \text{SuggestParameters}(D_{\theta})$ \Comment{Black-box optimizer proposes new parameter}
\For{$j \in [n_d]$}
\State $\phi_j \sim p(\Phi ; \theta_i)$
\State $x^j_{train}, x^j_{test}, y^j_{train}, y^j_{test} \sim p(\phi_j)$
\State $l_j(g_{\mathbf{W}}) \gets \mathcal{L}_{PFN}(g_{\mathbf{W}}, x^j_{train}, y^j_{train}, x^j_{test}, y^j_{test})$
\State $l_b \gets \infty$
\For{$f_k \in \{f_1, \dots, f_e\}$}
% Emphasize here that some (most) of these models need to be fit
\State $l_j(f_k) \gets \mathcal{L}_{PFN}(f_k, x^j_{train}, y^j_{train}, x^j_{test}, y^j_{test})$
\State $l_b \gets \min(l_b, l_j(f_k))$
\EndFor
\State $\delta_i^j \gets l_j(g_{\mathbf{W}}) - l_b$ \Comment{Estimate optimality gap}
\EndFor
\State $\delta_i \gets \frac{1}{n_d}\sum_{j=1}^{n_d}\delta_i^j$
\State $D_{\theta} \gets D_{\theta} \cup \{(\theta_i, \delta_i)\}$
\EndFor
\State \textbf{Return} Parameters and estimated optimality gaps $D_{\theta}$
\end{algorithmic}
\end{algorithm}

\section{SCM Parameters}
\label{appendix:scm}
We summarize the parameter space used in our experiments in Table \ref{table:params}. We show the parameters corresponding to the maximum optimality gap found over the $n_{trials}$ per max-min iteration in Table \ref{table:dist-params}. For this work, we implemented SCMs as randomized MLPs. In addition to the structure of the MLP itself, we can control more general properties such as the number of features (node activations) that are included in the dataset, and the ratio of features which are categorical. An illustration of this implementation is shown in Figure \ref{fig:scm} . For integer-valued parameters, we sample from the truncated normal distribution over the range specified in the table, and then round to the nearest integer. For categorical variables, we use an $\epsilon\text{-greedy}$ sampling scheme. For example, we select the activation function $\mu_a$ with probability $1-\epsilon$ and with probability $\epsilon$ select another activation function at random. For missingness, we mask the training data completely at random. During the maximizataion stage, when we perform the parameter search, we make a few simplifications to reduce the size of the search space. We discretize the values for all numeric parameters. Rather than varying $\mu_h, \mu_l, \mu_{x_{in}}$ independently, we couple them and define five different MLP sizes $(\mu_h, \mu_l, \mu_{x_{in}}) = [(5,3,2),(10,5,3), (32,8,3),(64,10,8),(128,12,13)]$ to search over. The parameters for our experiments is then $\theta = (\mu_h, \mu_l, \mu_{x_{in}}, \mu_a, \mu_{z_x}, \mu_c, \mu_{r_{cat}}, \mu_{r_s}, \mu_m, \mu_d)$, and each parameter is sampled independently. Within our framework, we can easily add additional parameters, including those that control the variance or completely separate dimensions such as the type of SCM (MLP or tree-based).

\begin{figure}
    \centering
    \includegraphics[width=\linewidth]{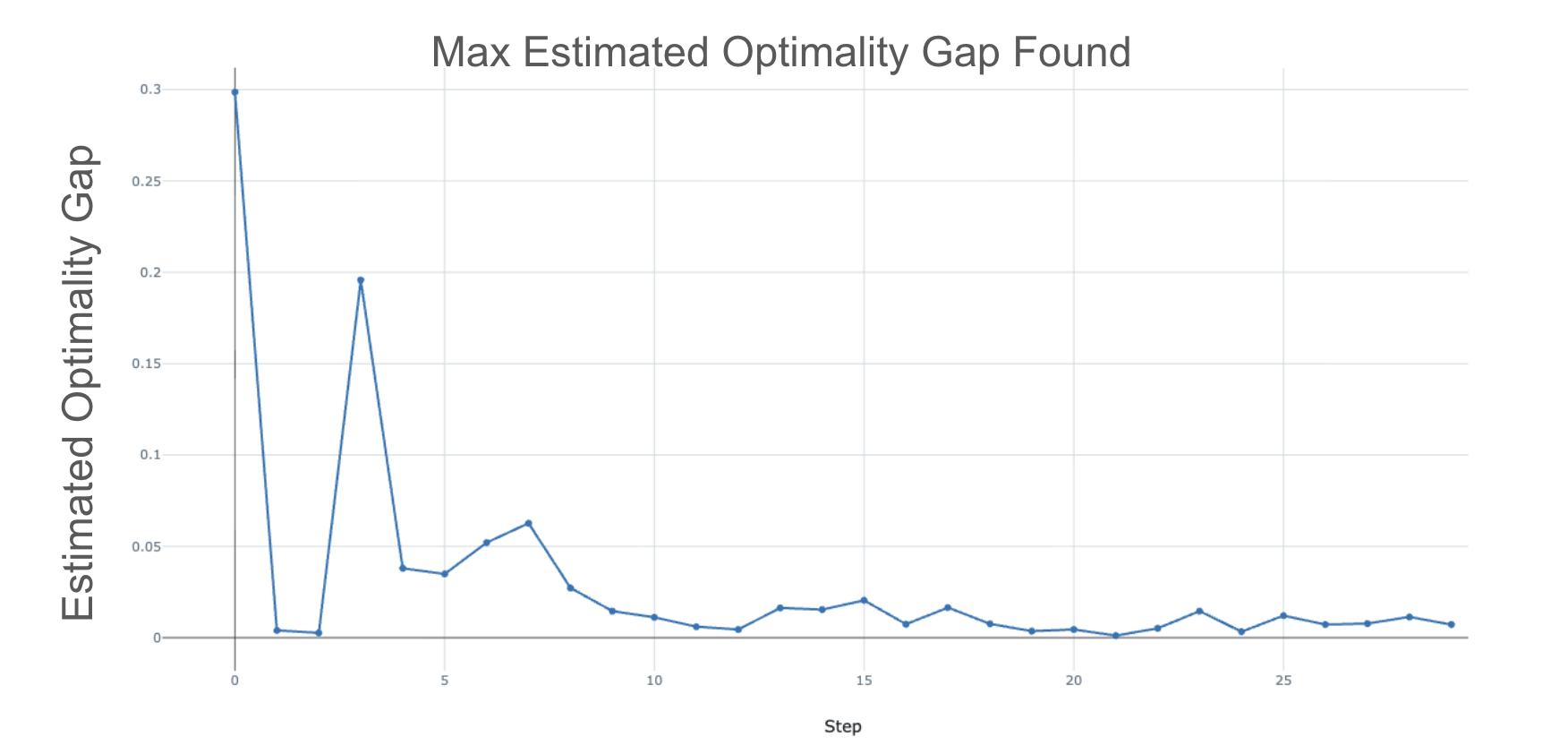}
    \caption{Maximum estimated optimality gap after each parameter search during model training. After epoch 5, the original TFM is introduced as an additional baseline model.}
    \label{fig:opt-gap}
\end{figure}

\begin{figure}[]
    \centering
    \includegraphics[width=\linewidth]{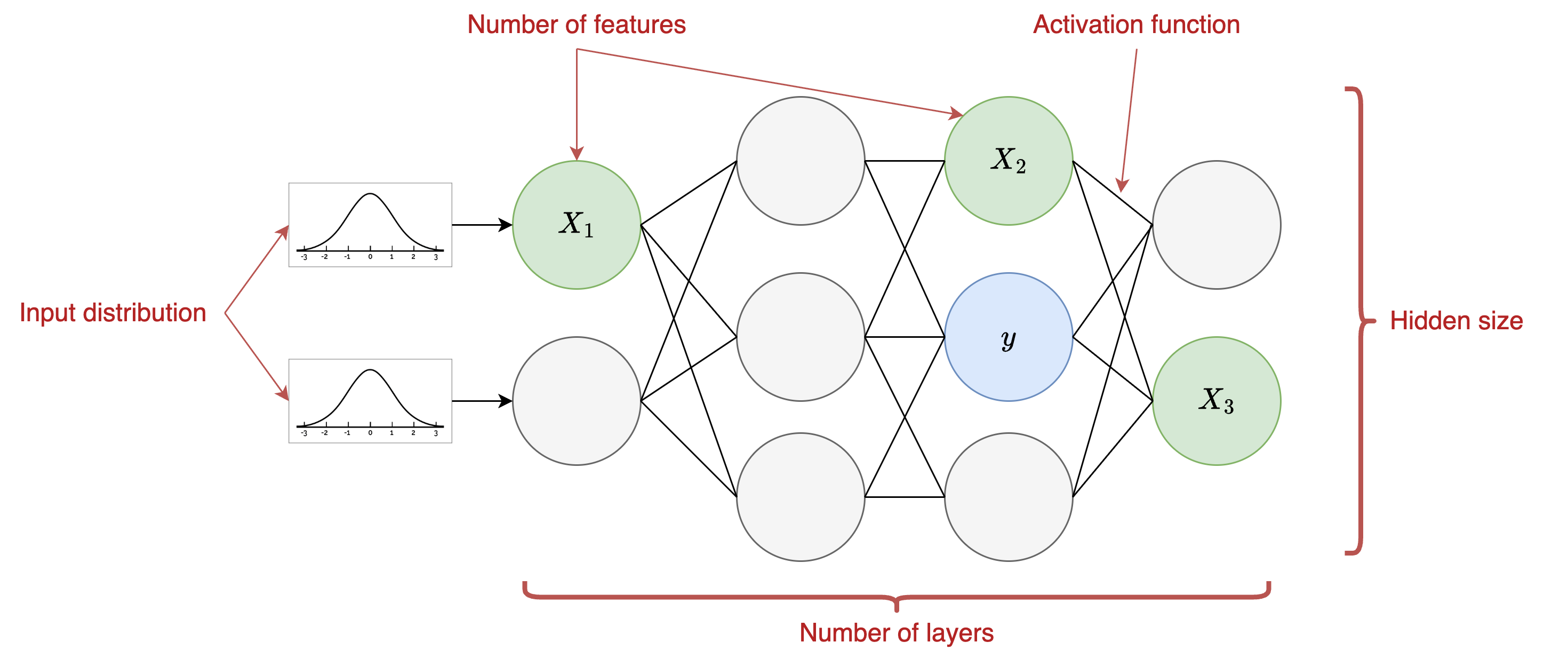}
    \caption{Anatomy of an MLP-based SCM.}
    \label{fig:scm}
\end{figure}

\begin{table*}[t]\label{table:params}
\centering
\caption{SCM Parameter space and sampling distributions. \label{table:params}}
\begin{adjustbox}{max width=1\textwidth}
\begin{tabular}{lccc}
\toprule
\textbf{Parameter} & \textbf{Type} & \textbf{Range} & \textbf{Distribution} \\
\midrule
Mean Hidden Size $\mu_h$ & integer & [3, 256] & $\text{TruncNorm}(\mu_h, 10)$ \\
Mean Num. Layers $\mu_l$& integer & [1, 12] & $\text{TruncNorm}(\mu_l, 1)$ \\
Mean Num. Inputs $\mu_{x_{in}}$ & integer & [1, 15] & $\text{TruncNorm}(\mu_{x_{in}}, 1)$\\
Activation fn $\mu_a$ & categorical & ["relu","elu","identity","tanh"] & $\epsilon\text{-greedy}$ ($\epsilon = 0.3$)\\
Mean Num. Features $\mu_{z_x}$ & integer & [2, 200] & $\text{TruncNorm}(\mu_{z_x}, \mu_{z_x}/4)$\\
Mean Num. Classes $\mu_{c}$ & integer & [2,10] & $\text{TruncNorm}(\mu_m, 1)$\\
Mean categorical ratio $\mu_{r_{cat}}$ & continuous & [0, 1] & $\text{TruncNorm}(\mu_{r_{cat}}, 0.1)$\\
Mean ordered cat. ratio $\mu_{r_s}$ & continuous & [0,1] & $\text{TruncNorm}(\mu_{r_s}, 0.1)$\\
Mean ratio missing $\mu_m$ & continuous & [0,1] & $\text{TruncNorm}(\mu_m, 0.1)$\\
Input distribution $\mu_d$ & categorical & ["exponential", "uniform", "normal"] & $\epsilon\text{-greedy}$ ($\epsilon = 0.3$)\\
\bottomrule
\end{tabular}
\end{adjustbox}
\end{table*}

\section{Proof of Optimal Maximizer Distribution}\label{sec:proofs}
In this section, we provide the proof that the optimal distribution for the maximizer in equation \ref{eq:dro-objective} is a softmax distribution such that $q_i^* \propto \exp(\eta \cdot \delta_{\theta_i}(\mathbf{W}))$.

\begin{proposition}
Let $\delta =(\delta_{\theta_1}, \dots, \delta_{\theta_n})\in\mathbb{R}^n$ and consider
\[
\max_{Q\in\Delta_n}\ \sum_{i=1}^n q_i\delta_{\theta_i}
\quad\text{s.t.}\quad
H(Q)\ge H_{\min},
\]
where $Q = (q_1,\dots, q_n)$, $\Delta_n=\{Q\in\mathbb{R}^n:\ q_i\ge0,\ \sum_i q_i=1\}$ and $H(Q)=-\sum_i q_i\log q_i$. If $0<H_{\min}<\log n$ and $\exists i,j$ such that $\delta_{\theta_i} \neq \delta_{\theta_j}$, the optimizer $q^*$ is strictly positive and there exists a unique $\lambda>0$ such that
\[
q_i^*=\frac{\exp(\delta_{\theta_i}/\lambda)}{\sum_{j=1}^n\exp(\delta_{\theta_j}/\lambda)},\qquad i=1,\dots,n,
\]
where $\lambda$ is determined implicitly by the entropy constraint $H(Q^*)=H_{\min}$.
\end{proposition}

\begin{proof}
Defining $ F(Q)=-\sum_i q_i\delta_{\theta_i}$, we rewrite the problem as
\[
\min_{Q\in\Delta_n}F(Q)\ 
\quad\text{s.t.}\quad g(Q):=-H(Q)+H_{\min}\le0.
\]
Since $F$ and $g$ are convex and the feasible region has interior points (e.g., the uniform distribution), Slater’s condition holds, so KKT conditions are necessary and sufficient.

The Lagrangian is given by
\begin{align*}
\mathcal{L}(Q,\lambda,\mu,\gamma) = & \\
& -\sum_i q_i\delta_{\theta_i} + \lambda\!\left(\sum_i q_i\log q_i + H_{\min}\right)\\
& + \mu\!\left(\sum_i q_i-1\right) - \sum_i\gamma_i q_i,\\
\end{align*}
where $\lambda\ge0$, $\mu\in\mathbb{R}$, and $\gamma_i\ge0$. Stationarity gives
\[
-\delta_{\theta_i} + \lambda(1+\log q_i) + \mu - \gamma_i = 0.
\]
When the entropy constraint is active ($\lambda>0$) and $q_i>0$, complementary slackness implies $\gamma_i=0$, yielding
\[
\lambda(1+\log q_i)=\delta_{\theta_i}-\mu.
\]
Exponentiating and enforcing $\sum_i q_i=1$ gives
\[
q_i(\lambda)=\frac{\exp(\delta_{\theta_i}/\lambda)}{\sum_j\exp(\delta_{\theta_j}/\lambda)}.
\]
Because $\lambda>0$, each $q_i>0$ and all $\gamma_i=0$, so the solution is interior. We now prove the uniqueness of  $\lambda$ solving $H(Q(\lambda)) = H_{min}$.  Let \(q_i(\lambda)=\dfrac{\exp(\delta_{\theta_i}/\lambda)}{Z(\lambda)}\), \(Z(\lambda)=\sum_j\exp(\delta_{\theta_j}/\lambda)\),
and write \(E(\lambda)=\mathbb{E}_{Q(\lambda)}[\delta]=\sum_i q_i(\lambda)\delta_{\theta_i}\),
\(E_2(\lambda)=\mathbb{E}_{q(\lambda)}[\delta^2]=\sum_i q_i(\lambda)\delta_{\theta_i}^2\).
The entropy of \(q(\lambda)\) can be written
\[
H(q(\lambda))=\log Z(\lambda)-\frac{E(\lambda)}{\lambda}.
\]
Differentiate \( \log Z \) with respect to $\lambda$, we have
\[
\frac{d}{d\lambda}\log Z(\lambda)
= \frac{1}{Z(\lambda)}\sum_i \exp(\delta_{\theta_i}/\lambda)\cdot\big(-\delta_{\theta_i}/\lambda^2\big)
= -\frac{E(\lambda)}{\lambda^2}.
\]
Next compute \(dE/d\lambda\). From \(\log q_i(\lambda)=\delta_{\theta_i}/\lambda-\log Z(\lambda)\) we get
\[
\frac{d q_i}{d\lambda}
= q_i\frac{d}{d\lambda}\log q_i
= q_i\Big(-\frac{\delta_{\theta_i}}{\lambda^2}-\frac{d}{d\lambda}\log Z(\lambda)\Big).
\]
Using this, we find
\begin{align*}
\frac{dE}{d\lambda}
& = \sum_i \delta_{\theta_i} \frac{d q_i}{d\lambda}\\
&= -\frac{1}{\lambda^2}\sum_i q_i\delta_{\theta_i}^2 - \frac{d}{d\lambda}\log Z(\lambda)\sum_i q_i\delta_{\theta_i}\\
& = -\frac{E_2(\lambda)}{\lambda^2} - \Big(-\frac{E(\lambda)}{\lambda^2}\Big)E(\lambda)\\
&= -\frac{E_2(\lambda)-E(\lambda)^2}{\lambda^2}.
\end{align*}
Therefore
\[
\frac{dE}{d\lambda} = -\frac{\operatorname{Var}_{Q(\lambda)}(\delta)}{\lambda^2}.
\]
Finally, differentiate \(H(Q(\lambda))=\log Z(\lambda)-E(\lambda)/\lambda\):
\[
\frac{d}{d\lambda}H(Q(\lambda))
= \frac{d}{d\lambda}\log Z(\lambda) - \frac{1}{\lambda}\frac{dE}{d\lambda} + \frac{E(\lambda)}{\lambda^2}.
\]
Substituting \(\dfrac{d}{d\lambda}\log Z(\lambda)=-\dfrac{E(\lambda)}{\lambda^2}\) cancels the last term, leaving
\[
\frac{d}{d\lambda}H(Q(\lambda)) = -\frac{1}{\lambda}\frac{dE}{d\lambda}
= \frac{\operatorname{Var}_{Q(\lambda)}(\delta)}{\lambda^3}.
\]
Since \(\operatorname{Var}_{Q(\lambda)}(\delta)\ge0\) and \(\lambda>0\), the derivative is nonnegative and is strictly positive unless all \(\delta_{\theta_i}\) are equal. Hence \(\lambda\mapsto H(Q(\lambda))\) is strictly increasing on \((0,\infty)\), which implies the uniqueness of \(\lambda\) solving \(H(q(\lambda))=H_{\min}\). By KKT sufficiency, $Q^*=Q(\lambda)$ is the global optimum.
\end{proof}

\section{Full Experiment Results}\label{appendix:results}
In Tables \ref{table:tabpertnet} and \ref{table:tabarena} we provide the full experiment results from the TabPertNet \cite{ye_towards_2024} and TabArena \cite{erickson_tabarena_2025} tabular benchmarks. For both benchmarks, we evaluate on classification datasets only, and limit our evaluation to datasets with at most $n=10k$ samples and $p=500$ features. We hold out some datasets from both benchmarks as validation datasets during model training. We report our test results on $N=74$ datasets from TabPertNet and $N=21$ datasets from TabArena. For TabArena, we use the same folds and repeats as defined in the original work and corresponding leaderboard. For TabPertNet, we use the same train/test split used in the original paper. We report AUC for TabPertNet and AUC one-vs-one (OVO) for TabArena.

\begin{table*}[h]
\centering
\caption{TabPertNet \cite{ye_towards_2024} Test Results (AUC $\uparrow$).\label{table:tabpertnet}}
\setlength{\tabcolsep}{2pt} % reduce horizontal padding
% \small % smaller font size
\begin{adjustbox}{width=1\textwidth,height=4.8in} % scale table if still too wide
\begin{tabular}{cccccccc}
\toprule
 \textbf{Dataset} & \textbf{TabPFN\textsubscript{n.e.} (RTFM)} & \textbf{TabPFN\textsubscript{n.e.} (Base)} & \textbf{CatBoost} & \textbf{Random Forest} & \textbf{XGBoost} & \textbf{Log. Reg.} & \textbf{MLP} \\
\midrule
Breast\_Cancer & 0.886 & 0.885 & 0.830 & 0.820 & 0.773 & 0.852 & 0.816 \\
1736\_combined-wine-data & 1.000 & 1.000 & 1.000 & 1.000 & 1.000 & 1.000 & 1.000 \\
0446\_newton\_hema & 0.976 & 0.952 & 0.976 & 0.976 & 1.000 & 1.000 & 0.952 \\
1564\_Mammographic-Mass-Data-Set & 0.938 & 0.934 & 0.843 & 0.919 & 0.857 & 0.916 & 0.892 \\
0408\_pharynx & 0.911 & 0.800 & 0.778 & 0.844 & 0.867 & 0.867 & 0.889 \\
1600\_VulNoneVul & 0.594 & 0.554 & 0.693 & 0.546 & 0.640 & 0.637 & 0.549 \\
Customer\_Behaviour & 0.991 & 0.994 & 0.734 & 0.818 & 0.837 & 0.855 & 0.852 \\
1333\_ricci\_vs\_destefano & 1.000 & 1.000 & 0.708 & 0.917 & 1.000 & 0.875 & 0.875 \\
1635\_Is-this-a-good-customer & 0.789 & 0.796 & 0.750 & 0.767 & 0.788 & 0.693 & 0.639 \\
1408\_national-longitudinal-survey-binary & 1.000 & 1.000 & 0.991 & 0.988 & 0.962 & 0.995 & 0.998 \\
b\_depressed & 0.551 & 0.541 & 0.513 & 0.567 & 0.606 & 0.446 & 0.614 \\
0356\_delta\_elevators & 0.951 & 0.953 & 0.950 & 0.955 & 0.947 & 0.945 & 0.945 \\
0437\_quake & 0.518 & 0.504 & 0.496 & 0.507 & 0.480 & 0.603 & 0.606 \\
1461\_heart-failure & 0.856 & 0.867 & 0.856 & 0.833 & 0.839 & 0.944 & 0.867 \\
0124\_analcatdata\_impeach & 1.000 & 1.000 & 0.550 & 0.750 & 0.850 & 0.900 & 1.000 \\
2304\_electricity & 0.891 & 0.890 & 0.870 & 0.856 & 0.869 & 0.699 & 0.790 \\
0446\_arsenic-female-bladder & 0.811 & 0.819 & 0.776 & 0.825 & 0.774 & 0.825 & 0.887 \\
0948\_Ishwar & 0.997 & 0.999 & 0.961 & 0.966 & 0.968 & 0.961 & 0.973 \\
2393\_airlines & 0.650 & 0.661 & 0.572 & 0.645 & 0.592 & 0.589 & 0.567 \\
0400\_analcatdata\_supreme & 1.000 & 1.000 & 1.000 & 1.000 & 1.000 & 0.996 & 0.999 \\
2619\_sf-police-incidents & 0.584 & 0.515 & 0.572 & 0.449 & 0.407 & 0.454 & 0.477 \\
2308\_electricity & 0.920 & 0.913 & 0.893 & 0.885 & 0.895 & 0.732 & 0.842 \\
0885\_compas-two-years & 0.744 & 0.746 & 0.692 & 0.735 & 0.659 & 0.729 & 0.733 \\
1201\_Gender-Recognition-by-Voice & 1.000 & 0.999 & 0.998 & 0.999 & 0.999 & 0.995 & 0.999 \\
0546\_analcatdata\_seropositive & 1.000 & 1.000 & 0.833 & 0.861 & 1.000 & 0.972 & 0.944 \\
0345\_delta\_ailerons & 0.977 & 0.976 & 0.972 & 0.973 & 0.976 & 0.967 & 0.971 \\
0445\_arsenic-male-bladder & 0.983 & 0.991 & 0.784 & 0.991 & 0.957 & 0.983 & 0.595 \\
2389\_airlines & 0.678 & 0.699 & 0.663 & 0.646 & 0.637 & 0.656 & 0.588 \\
0424\_autoPrice & 1.000 & 1.000 & 0.958 & 1.000 & 1.000 & 0.938 & 0.938 \\
0419\_pm10 & 0.628 & 0.591 & 0.565 & 0.637 & 0.675 & 0.404 & 0.537 \\
2703\_compas-two-years & 0.763 & 0.761 & 0.731 & 0.756 & 0.704 & 0.752 & 0.749 \\
1774\_Early-Stage-Diabetes & 1.000 & 1.000 & 1.000 & 1.000 & 1.000 & 0.997 & 0.999 \\
1006\_Titanic & 0.742 & 0.742 & 0.742 & 0.742 & 0.742 & 0.662 & 0.742 \\
1752\_Wisconsin-breast-cancer & 1.000 & 1.000 & 0.997 & 0.999 & 0.999 & 1.000 & 1.000 \\
0541\_plasma\_retinol & 0.688 & 0.693 & 0.315 & 0.562 & 0.594 & 0.506 & 0.574 \\
1011\_cleve & 0.913 & 0.909 & 0.856 & 0.928 & 0.923 & 0.918 & 0.952 \\
0472\_analcatdata\_marketing & 0.714 & 0.663 & 0.566 & 0.480 & 0.437 & 0.383 & 0.714 \\
0447\_arsenic-female-lung & 0.948 & 0.888 & 0.784 & 0.888 & 0.970 & 0.276 & 0.293 \\
2391\_airlines & 0.616 & 0.628 & 0.665 & 0.661 & 0.629 & 0.608 & 0.558 \\
2392\_airlines & 0.612 & 0.601 & 0.665 & 0.627 & 0.647 & 0.691 & 0.696 \\
0555\_socmob & 0.996 & 0.998 & 0.951 & 0.957 & 0.960 & 0.932 & 0.959 \\
2622\_sf-police-incidents & 0.502 & 0.456 & 0.398 & 0.474 & 0.561 & 0.504 & 0.469 \\
0406\_visualizing\_environmental & 1.000 & 1.000 & 1.000 & 1.000 & 1.000 & 1.000 & 1.000 \\
diabetes\_data\_upload & 1.000 & 0.998 & 1.000 & 1.000 & 1.000 & 0.989 & 0.992 \\
0312\_cpu\_act & 0.981 & 0.981 & 0.976 & 0.977 & 0.971 & 0.956 & 0.979 \\
1458\_kdd\_ipums\_la\_97-small & 0.942 & 0.942 & 0.932 & 0.943 & 0.930 & 0.881 & 0.540 \\
piracydataset & 0.678 & 0.604 & 0.684 & 0.737 & 0.656 & 0.770 & 0.778 \\
2620\_sf-police-incidents & 0.531 & 0.492 & 0.367 & 0.541 & 0.425 & 0.491 & 0.517 \\
2306\_electricity & 0.879 & 0.882 & 0.870 & 0.862 & 0.863 & 0.771 & 0.838 \\
2621\_sf-police-incidents & 0.611 & 0.579 & 0.634 & 0.350 & 0.628 & 0.630 & 0.516 \\
2390\_airlines & 0.600 & 0.597 & 0.590 & 0.574 & 0.560 & 0.584 & 0.516 \\
Bank\_Personal\_Loan\_Modelling & 0.589 & 0.582 & 0.605 & 0.594 & 0.606 & 0.604 & 0.562 \\
1512\_eye\_movements & 0.652 & 0.650 & 0.701 & 0.686 & 0.700 & 0.594 & 0.660 \\
NFL & 0.632 & 0.595 & 1.000 & 1.000 & 1.000 & 1.000 & 0.977 \\
2305\_electricity& 0.928 & 0.931 & 0.920 & 0.903 & 0.911 & 0.788 & 0.858 \\
1592\_Diabetes-Data-Set & 0.882 & 0.882 & 0.871 & 0.884 & 0.874 & 0.839 & 0.844 \\
0509\_pollen & 0.472 & 0.467 & 0.451 & 0.481 & 0.485 & 0.478 & 0.502 \\
1413\_shill-bidding & 1.000 & 1.000 & 1.000 & 1.000 & 0.999 & 0.999 & 0.996 \\
BankNoteAuthentication & 1.000 & 1.000 & 1.000 & 1.000 & 1.000 & 0.999 & 1.000 \\
0284\_bank8FM & 0.992 & 0.991 & 0.989 & 0.991 & 0.988 & 0.991 & 0.991 \\
0292\_cpu\_small & 0.980 & 0.979 & 0.978 & 0.979 & 0.974 & 0.949 & 0.972 \\
1742\_Loan-Predication & 0.729 & 0.698 & 0.669 & 0.671 & 0.703 & 0.692 & 0.657 \\
Employee Satisfaction Index & 0.574 & 0.538 & 0.463 & 0.484 & 0.522 & 0.468 & 0.429 \\
1759\_Red--White-wine-Dataset & 1.000 & 1.000 & 1.000 & 1.000 & 1.000 & 1.000 & 1.000 \\
0526\_colleges\_aaup & 1.000 & 1.000 & 0.918 & 0.913 & 0.939 & 0.958 & 0.960 \\
bt\_dataset\_t3 & 1.000 & 1.000 & 1.000 & 1.000 & 1.000 & 0.970 & 0.997 \\
0435\_strikes & 1.000 & 1.000 & 1.000 & 0.995 & 0.991 & 0.616 & 0.903 \\
UniversalBank & 0.589 & 0.582 & 0.605 & 0.594 & 0.609 & 0.604 & 0.579 \\
1692\_Gender-Classification-Dataset & 0.996 & 0.997 & 0.994 & 0.996 & 0.996 & 0.996 & 0.996 \\
1142\_Sick\_numeric & 0.999 & 0.998 & 0.999 & 0.999 & 0.997 & 0.949 & 0.943 \\
1451\_early-stage-diabetes & 1.000 & 1.000 & 1.000 & 1.000 & 1.000 & 0.997 & 1.000 \\
TravelInsurancePrediction & 0.823 & 0.811 & 0.812 & 0.802 & 0.812 & 0.786 & 0.815 \\
1898\_Personal-Loan-Modeling & 0.590 & 0.611 & 0.633 & 0.630 & 0.633 & 0.627 & 0.601 \\
new\_model & 1.000 & 1.000 & 1.000 & 1.000 & 1.000 & 1.000 & 1.000 \\
1578\_kdd\_ipums\_la\_97-small & 0.957 & 0.952 & 0.947 & 0.958 & 0.952 & 0.906 & 0.534 \\
loan\_train & 0.731 & 0.672 & 0.657 & 0.682 & 0.695 & 0.671 & 0.645 \\
\bottomrule
\end{tabular}
\end{adjustbox}
\end{table*}

\begin{table*}[h]\label{table:tabarena}
\centering
\caption{TabArena \cite{erickson_tabarena_2025} Test Results (AUC OVO $\uparrow$). \label{table:tabarena}}
\setlength{\tabcolsep}{3pt} % reduce horizontal padding
% \small % smaller font size
\begin{adjustbox}{max width=1\textwidth} % scale table if still too wide
\begin{tabular}{cccccccc}
\toprule
 \textbf{Dataset} & \textbf{TabPFN\textsubscript{n.e.} (RTFM)} & \textbf{TabPFN\textsubscript{n.e.} (Base)} & \textbf{CatBoost} & \textbf{Random Forest} & \textbf{XGBoost} & \textbf{Log. Reg.} & \textbf{MLP} \\
\midrule
Bank\_Customer\_Churn & $0.861 \pm 0.007$ & $0.870 \pm 0.007$ & $0.843 \pm 0.008$ & $0.848 \pm 0.007$ & $0.831 \pm 0.007$ & $0.838 \pm 0.008$ & $0.798 \pm 0.009$  \\
Fitness\_Club & $0.821 \pm 0.012$ & $0.821 \pm 0.012$ & $0.748 \pm 0.020$ & $0.745 \pm 0.020$ & $0.718 \pm 0.021$ & $0.780 \pm 0.018$ & $0.701 \pm 0.022$  \\
Is-this-a-good-customer & $0.743 \pm 0.018$ & $0.738 \pm 0.023$ & $0.677 \pm 0.025$ & $0.725 \pm 0.019$ & $0.679 \pm 0.021$ & $0.677 \pm 0.024$ & $0.618 \pm 0.029$  \\
MIC & $0.670 \pm 0.011$ & $0.660 \pm 0.012$ & $0.674 \pm 0.019$ & $0.671 \pm 0.020$ & $0.667 \pm 0.022$ & $0.700 \pm 0.025$ & $0.659 \pm 0.020$  \\
Marketing\_Campaign & $0.902 \pm 0.016$ & $0.900 \pm 0.017$ & $0.834 \pm 0.017$ & $0.837 \pm 0.021$ & $0.819 \pm 0.020$ & $0.848 \pm 0.019$ & $0.810 \pm 0.018$  \\
NATICUSdroid & $0.984 \pm 0.002$ & $0.983 \pm 0.002$ & $0.984 \pm 0.002$ & $0.985 \pm 0.002$ & $0.976 \pm 0.003$ & $0.981 \pm 0.002$ & $0.978 \pm 0.003$  \\
anneal & $0.994 \pm 0.010$ & $0.999 \pm 0.003$ & $0.983 \pm 0.023$ & $0.993 \pm 0.011$ & $0.995 \pm 0.011$ & $0.995 \pm 0.007$ & $0.998 \pm 0.003$  \\
blood-transfusion-service-center & $0.749 \pm 0.029$ & $0.748 \pm 0.027$ & $0.673 \pm 0.033$ & $0.738 \pm 0.029$ & $0.681 \pm 0.026$ & $0.752 \pm 0.027$ & $0.706 \pm 0.070$  \\
churn & $0.926 \pm 0.012$ & $0.926 \pm 0.010$ & $0.920 \pm 0.009$ & $0.923 \pm 0.012$ & $0.916 \pm 0.008$ & $0.820 \pm 0.013$ & $0.828 \pm 0.012$  \\
coil2000\_insurance\_policies & $0.766 \pm 0.012$ & $0.757 \pm 0.013$ & $0.711 \pm 0.013$ & $0.748 \pm 0.012$ & $0.691 \pm 0.014$ & $0.714 \pm 0.011$ & $0.682 \pm 0.012$  \\
credit-g & $0.779 \pm 0.020$ & $0.766 \pm 0.017$ & $0.755 \pm 0.023$ & $0.770 \pm 0.020$ & $0.771 \pm 0.020$ & $0.761 \pm 0.023$ & $0.723 \pm 0.027$  \\
diabetes & $0.838 \pm 0.023$ & $0.840 \pm 0.024$ & $0.796 \pm 0.020$ & $0.833 \pm 0.024$ & $0.820 \pm 0.023$ & $0.829 \pm 0.024$ & $0.721 \pm 0.035$  \\
hazelnut-spread-contaminant-detection & $0.986 \pm 0.003$ & $0.986 \pm 0.004$ & $0.974 \pm 0.005$ & $0.967 \pm 0.005$ & $0.957 \pm 0.006$ & $0.608 \pm 0.047$ & $0.849 \pm 0.014$  \\
maternal\_health\_risk & $0.945 \pm 0.013$ & $0.941 \pm 0.012$ & $0.939 \pm 0.014$ & $0.935 \pm 0.014$ & $0.941 \pm 0.014$ & $0.800 \pm 0.018$ & $0.777 \pm 0.024$  \\
polish\_companies\_bankruptcy & $0.956 \pm 0.006$ & $0.955 \pm 0.007$ & $0.952 \pm 0.007$ & $0.958 \pm 0.010$ & $0.933 \pm 0.006$ & $0.552 \pm 0.041$ & $0.650 \pm 0.087$  \\
qsar-biodeg & $0.938 \pm 0.012$ & $0.934 \pm 0.012$ & $0.923 \pm 0.013$ & $0.929 \pm 0.012$ & $0.928 \pm 0.012$ & $0.925 \pm 0.013$ & $0.917 \pm 0.013$  \\
seismic-bumps & $0.771 \pm 0.023$ & $0.768 \pm 0.029$ & $0.661 \pm 0.024$ & $0.746 \pm 0.016$ & $0.676 \pm 0.016$ & $0.713 \pm 0.026$ & $0.639 \pm 0.029$  \\
splice & $0.995 \pm 0.002$ & $0.995 \pm 0.002$ & $0.994 \pm 0.002$ & $0.994 \pm 0.002$ & $0.994 \pm 0.002$ & $0.989 \pm 0.002$ & $0.991 \pm 0.002$  \\
students\_dropout\_and\_academic\_success & $0.883 \pm 0.007$ & $0.880 \pm 0.005$ & $0.866 \pm 0.006$ & $0.859 \pm 0.007$ & $0.864 \pm 0.009$ & $0.844 \pm 0.005$ & $0.812 \pm 0.007$  \\
taiwanese\_bankruptcy\_prediction & $0.944 \pm 0.007$ & $0.943 \pm 0.005$ & $0.932 \pm 0.010$ & $0.937 \pm 0.007$ & $0.923 \pm 0.012$ & $0.576 \pm 0.027$ & $0.521 \pm 0.015$  \\
website\_phishing & $0.965 \pm 0.007$ & $0.974 \pm 0.005$ & $0.972 \pm 0.005$ & $0.972 \pm 0.005$ & $0.969 \pm 0.005$ & $0.923 \pm 0.009$ & $0.969 \pm 0.006$  \\
\bottomrule
\end{tabular}
\end{adjustbox}
\end{table*}

\begin{table*}[h]\label{table:dist-params}
\centering
\caption{Mean parameter values corresponding to the maximum estimated optimality gap for each max-min iteration of the RTFM algorithm. We observe significant variation in these parameters across iterations.}\label{table:dist-params}
\begin{tabular}{|c|cccccccccc|}
\hline
max-min iter. & $\mu_l$ & $\mu_h$ & $\mu_{x_{in}}$ & $\mu_{z_x}$ & $\mu_{r_{s}}$ & $\mu_c$ & $\mu_a$ & $\mu_{r_{cat}}$ & $\mu_{m}$ & $\mu_d$ \\
\hline
1  & 5 & 10 & 3 & 5 & 1.0 & 8 & tanh & 0.9 & 0.4 & uniform \\
2  & 3 & 5 & 2 & 5 & 0.1 & 2 & identity & 1.0 & 0.4 & uniform \\
3  & 12 & 128 & 13 & 200 & 0.0 & 2 & identity & 0.4 & 0.0 & uniform \\
4  & 3 & 5 & 2 & 6 & 0.6 & 2 & tanh & 1.0 & 0.0 & exponential \\
5  & 10 & 64 & 8 & 32 & 0.4 & 2 & tanh & 1.0 & 0.0 & uniform \\
6  & 8 & 32 & 5 & 13 & 0.8 & 10 & identity & 0.0 & 0.2 & normal \\
7  & 10 & 64 & 8 & 162 & 0.9 & 6 & elu & 0.6 & 0.0 & normal \\
8  & 5 & 10 & 3 & 13 & 0.0 & 10 & tanh & 0.8 & 0.0 & uniform \\
9  & 5 & 10 & 3 & 5 & 0.3 & 6 & identity & 0.6 & 0.0 & exponential \\
10 & 5 & 10 & 3 & 5 & 0.1 & 10 & tanh & 0.5 & 0.0 & normal \\
11 & 3 & 5 & 2 & 5 & 0.0 & 8 & relu & 1.0 & 0.0 & normal \\
12 & 3 & 5 & 2 & 5 & 0.6 & 4 & relu & 0.6 & 0.0 & exponential \\
13 & 3 & 5 & 2 & 5 & 0.4 & 6 & tanh & 0.6 & 0.0 & normal \\
14 & 5 & 10 & 3 & 5 & 0.8 & 8 & identity & 0.8 & 0.0 & exponential \\
15 & 3 & 5 & 2 & 5 & 1.0 & 8 & elu & 0.6 & 0.0 & uniform \\
16 & 3 & 5 & 2 & 6 & 0.7 & 6 & identity & 0.3 & 0.0 & uniform \\
17 & 8 & 32 & 5 & 91 & 0.0 & 2 & tanh & 1.0 & 0.0 & exponential \\
18 & 3 & 5 & 2 & 5 & 0.0 & 8 & tanh & 1.0 & 0.0 & exponential \\
19 & 3 & 5 & 2 & 6 & 0.2 & 6 & tanh & 0.2 & 0.0 & exponential \\
20 & 12 & 128 & 13 & 200 & 0.2 & 2 & identity & 0.1 & 0.0 & exponential \\
21 & 3 & 5 & 2 & 6 & 0.9 & 8 & relu & 0.6 & 0.0 & exponential \\
22 & 3 & 5 & 2 & 5 & 0.3 & 8 & identity & 0.2 & 0.0 & uniform \\
23 & 5 & 10 & 3 & 5 & 0.2 & 4 & tanh & 0.3 & 0.0 & normal \\
24 & 3 & 5 & 2 & 5 & 0.8 & 8 & identity & 0.1 & 0.0 & uniform \\
25 & 3 & 5 & 2 & 5 & 0.2 & 4 & relu & 1.0 & 0.0 & exponential \\
26 & 3 & 5 & 2 & 5 & 0.5 & 8 & tanh & 0.8 & 0.0 & normal \\
27 & 8 & 32 & 5 & 13 & 0.1 & 4 & identity & 0.0 & 0.0 & exponential \\
28 & 3 & 5 & 2 & 6 & 0.4 & 4 & tanh & 0.1 & 0.0 & exponential \\
29 & 3 & 5 & 2 & 5 & 0.1 & 8 & identity & 0.2 & 0.0 & exponential \\
\hline
\end{tabular}
\end{table*}

% \bibliography{aaai2026}

\end{document}